\documentclass[sigconf]{acmart}
\renewcommand\footnotetextcopyrightpermission[1]{} 
\usepackage{booktabs} 
\usepackage{graphicx}  
\usepackage{epsfig}
\usepackage{subfigure}
\usepackage{bm}
\usepackage{multirow}
\usepackage{tabularx}
\usepackage{blkarray}
\usepackage{algorithm}
\usepackage{algorithmic}
\usepackage{amsmath,amssymb,amsthm}
\usepackage{enumitem}

\newtheorem{specialthm}{theorem}

\newtheorem{speciallma}{lemma}

\newcommand{\tabincell}[2]{\begin{tabular}{@{}#1@{}}#2\end{tabular}}

\acmDOI{10.475/123_4}

\begin{document}
\title{Online Compact Convexified Factorization Machine}

\author{Wenpeng Zhang$^\dag,\quad$ Xiao Lin$^\dag,\quad$ Peilin Zhao$^\ddag$}
\affiliation{%
 \institution{$^\dag$Tsinghua University,$\quad$ $^\ddag$South China University of Technology}
}
\email{zhangwenpeng0@gmail.com, jackielinxiao@gmail.com, peilinzhao@hotmail.com}

%
%



\begin{abstract}
Factorization Machine (FM) is a supervised learning approach with a powerful capability of feature engineering. It yields state-of-the-art performance in various batch learning tasks where all the training data is made available prior to the training. However, in real-world applications where the data arrives sequentially in a streaming manner, the high cost of re-training with batch learning algorithms has posed formidable challenges in the online learning scenario. The initial challenge is that no prior formulations of FM could fulfill the requirements in Online Convex Optimization (OCO) -- the paramount framework for online learning algorithm design. To address the aforementioned challenge, we invent a new convexification scheme leading to a Compact Convexified FM (CCFM) that seamlessly meets the requirements in OCO. However for learning Compact Convexified FM (CCFM) in the online learning setting, most existing algorithms suffer from expensive projection operations. To address this subsequent challenge, we follow the general projection-free algorithmic framework of Online Conditional Gradient and propose an Online Compact Convex Factorization Machine (OCCFM) algorithm that eschews the projection operation with efficient linear optimization steps. In support of the proposed OCCFM in terms of its theoretical foundation, we prove that the developed algorithm achieves a sub-linear regret bound. To evaluate the empirical performance of OCCFM, we conduct extensive experiments on 6 real-world datasets for online recommendation and binary classification tasks. The experimental results show that OCCFM outperforms the state-of-art online learning algorithms.
\end{abstract}

%
%
%



\maketitle

\section{Introduction}
\label{Introduction}
Factorization Machine (FM)~\cite{rendle2010factorization} is a generic approach for supervised learning . It provides an efficient mechanism for feature engineering, capturing first-order information of each input feature as well as second-order pairwise feature interactions with a low-rank matrix in a factorized form. FM achieves state-of-the-art performances in various applications, including recommendation~\cite{rendle2011fast,Nguyen:2014:GPF:2600428.2609623}, computational advertising~\cite{juan2017field}, search ranking~\cite{Lu:2017:MFM:3018661.3018716} and toxicogenomics prediction~\cite{Yamada:2017:CFM:3097983.3098103}, and so on. As such, FM has recently regained significant attention from researchers~\cite{Xu2016Synergies,blondel2016polynomial,blondel2016higher,ijcai2017-435} due to its increasing popularity in industrial applications and data science competitions~\cite{juan2017field,zhong2016scaling}.

Despite of the overwhelming research on Factorization Machine, majority of the existing studies are conducted in the batch learning setting where all the training data is available before training.
However, in many real-world scenarios, like online recommendation and online advertising~\cite{Wang:2013:OMC:2507157.2507176, ad41159}, the training data arrives sequentially in a streaming fashion.
If the batch learning algorithms are applied in accordance with the streams in such scenarios, the models have to be re-trained each time new data arrives. Since these data streams usually arrive in large-scales and are changing constantly in a real-time manner, the incurred high re-training cost makes batch learning algorithms impractical in such settings. This creates an impending need for an efficient online learning algorithm for Factorization Machine. Moreover, it is expected that the online learning algorithm has a theoretical guarantee for its performance.

In the current paper, we aim to develop an ideal algorithm based on the paramount framework -- Online Convex Optimization (OCO) ~\cite{Shalev-Shwartz:2012:OLO:2185819.2185820,hazan2016introduction} in the online learning setting. Unfortunately, extant formulations are unable to fulfill the two fundamental requirements demanded in OCO: ({\romannumeral1}) any instance of all the parameters should be represented as a single point from a convex compact decision set; and, ({\romannumeral2}) the loss incurred by the prediction should be formulated as a convex function over the decision set. Indeed, in most existing formulations for FM  \cite{rendle2010factorization, juan2017field, Lu:2017:MFM:3018661.3018716, cheng2014gradient, blondel2016higher}, the loss functions are non-convex with respect to the factorized feature interaction matrix, thus violating requirement ({\romannumeral2}). Further, although some studies have proposed formulations for convex FM which rectify the non-convexity problem \cite{blondel2015convex, Yamada:2017:CFM:3097983.3098103}, they still treat the feature weight vector and the feature interaction matrix as separated parameters, thus violating requirement ({\romannumeral1}) in OCO.

To address these problems, we propose a new convexification scheme for FM. Specifically, we rewrite the global bias, the feature weight vector and the feature interaction matrix into a compact augmented symmetric matrix and restrict the augmented matrix with a nuclear norm bound, which is a convex surrogate of the low-rank constraint~\cite{boyd2004convex}. Therefore the augmented matrices form a convex compact decision set which is essentially a symmetric bounded nuclear norm ball. Then we rewrite the prediction of FM into a convex linear function with respect to the augmented matrix, thus the loss incurred by the prediction is convex. Based on the convexification scheme, the resulting formulation of \textit{Compact Convexified FM}\,(CCFM) can seamlessly meet the aforementioned requirements of the OCO framework.


Yet, when we investigate various online learning algorithms for Compact Convexified Factorization Machine within the OCO framework, we find that most of existing online learning algorithms involve a projection step in every iteration. When the decision set is a bounded nuclear norm ball, the projection amounts to the computationally expensive Singular Value Decomposition (SVD), and consequently limits the applicability of most online learning algorithms. Notably, one exceptional algorithm is Online Conditional Gradient (OCG), which eschews the projection operation by a linear optimization step. When OCG is applied to the nuclear norm ball, the linear optimization amounts to the computation of maximal singular vectors of a matrix, which is much simpler. However, it remains theoretically unknown whether an algorithm similar to OCG still exists for the specific subset of a bounded nuclear norm ball.


In response, we propose an online learning algorithm for CCFM, which is named as \textit{Online Compact Convexified Factorization Machine}\,(OCCFM). We prove that when the decision set is a symmetric nuclear norm ball, the linear optimization needed for an OCG-alike algorithm still exsits, i.e. it amounts to the computation of the maximal singular vectors of a specific symmetric matrix. Based on this finding, we propose an OCG-alike online learning algorithm for OCCFM. As OCCFM is a variant of OCG, we prove that the theoretical analysis of OCG still fits for OCCFM, which achieves a sub-linear regret bound in order of $O(T^{3/4})$. Further, we conduct extensive experiments on real-world datasets to evaluate the empirical performance of OCCFM. As shown in the experimental results in both online recommendation and online binary classification tasks, OCCFM outperforms the state-of-art online learning algorithms in terms of both efficiency and prediction accuracy.


The contributions of this work are summarized as follows:
\begin{enumerate}
\item To the best of our knowledge, we are the first to propose the Online Compact Convexified Factorization Machine with theoretical guarantees, which is a new variant of FM in the online learning setting.
\item The proposed formulation of CCFM is superior to prior research in that it seamlessly fits into the online learning framework. Moreover, this formulation can be used in not only online setting, but also batch and stochastic settings.
\item We propose a routine for the linear optimization on the decision set of CCFM based on the Online Conditional Gradient algorithm, leading to an OCG-alike online learning algorithm for CCFM. Moreover, this finding also applies to any other online learning task whose decision set is the symmetric bounded nuclear norm ball.
\item We evaluate the performance of the proposed OCCFM on both online recommendation and online binary classification tasks, showing that OCCFM outperforms state-of-art online learning algorithms.
\end{enumerate}

We believe our work sheds light on the Factorization Machine research, especially for online Factorization Machine. Due to the wide application of FM, our work is of both theoretical and practical significance.


\section{Preliminaries}
\label{Preliminary}

In this section, we first review more details of factorization machine, and then provide an introduction to the online convex optimization framework and the online conditional gradient algorithm, both of which we utilize to design our online learning algorithm for factorization machine.

\subsection{Factorization Machine}

Factorization Machine\,(FM)~\cite{rendle2010factorization} is a general supervised learning approach working with any real-valued feature vector. The most important characteristic of FM is the capability of powerful feature engineering. In addition to the usual first-order information of each input feature, it captures the second-order pairwise feature interaction in modeling, which leads to stronger expressiveness than linear models.

Given an input feature vector $\bm{x}\in \mathbb{R}^{d}$, vanilla FM makes the prediction $\hat{y}\in \mathbb{R}$ with the following formula:
\begin{small}
\begin{equation*}
\hat{y}(\bm{x},\omega_0,\bm{\omega},\bm{V})=\omega_0+\bm{\omega}^{T}\bm{x}+\sum_{i=1}^d\sum_{j=i+1}^d(\bm{VV}^T)_{ij}x_ix_j,
\vspace{-10pt}
\end{equation*}
\end{small}\\
where $\omega_0\in \mathbb{R}$ is the bias term, $\bm{\omega}\in \mathbb{R}^{d}$ is the first-order feature weight vector and $\bm{V}\in\mathbb{R}^{d\times k}$ is the second-order factorized feature interaction matrix; $(\bm{VV}^T)_{ij}$ is the entry in the $i$-th row and $j$-th column of matrix $\bm{VV}^T$, and $k\ll d$ is the hyper-parameter determining the rank of $\bm{V}$. As shown in previous studies~\cite{blondel2015convex,Yamada:2017:CFM:3097983.3098103}, $\hat{y}$ is non-convex with respect to $\bm{V}$.

The non-convexity in vanilla FM will result in some problems in practice, such as local minima and instability of convergence. In order to overcome these problems, two formulations for convex FM~\cite{blondel2015convex,Yamada:2017:CFM:3097983.3098103} have been proposed, where the feature interaction is directly modeled as $\bm{Z}\in\mathbb{R}^{d\times d}$ rather than $\bm{VV}^T$ in the factorized form which induces non-convexity. The matrix $\bm{Z}$ is then imposed upon a bounded nuclear norm constraint to maintain the low-rank property. In general, convex FM allows for more general modeling of feature interaction and is quite effective in practice. The difference between the two formulations is whether the diagonal entries of $\bm{Z}$ are utilized in prediction. For clarity, we refer to the formulation in~\cite{blondel2015convex} as Convex FM\,$(1)$ and that in~\cite{Yamada:2017:CFM:3097983.3098103} as Convex FM\,$(2)$.

Note that we propose a new convexification scheme for FM in this work, which is inherently different from the above two formulations for convex FM. The details of our convexification scheme and its comparison with convex FMs will be presented in Section~\ref{Convexfication_Scheme}.

\subsection{Online Convex Optimization}

Online Convex Optimization\,(OCO)~\cite{Shalev-Shwartz:2012:OLO:2185819.2185820,hazan2016introduction} is the paramount framework for designing online learning algorithms. It can be seen as a structured repeated game between a learner and an adversary. At each round $t\in\{1,2,\cdots,T\}$, the learner is required to generate a decision point $\bm{x}_t$ from a convex compact set $\mathcal{Q}\subseteq \mathbb{R}^{n}$. Then the adversary replies the learner's decision with a convex loss function $f_{t}:~\mathcal{Q}\rightarrow\mathbb{R}$ and the learner suffers the loss $f_t(\bm{x}_t)$. The goal of the learner is to generate a sequence of decisions $\{\bm{x}_t|~t=1,2,\cdots,T\}$ so that the regret with respect to the best fixed decision in hindsight
\begin{small}
\begin{equation*}
\label{Eqn::Regret}
\mathbf{regret}_T=\sum_{t=1}^Tf_t(\bm{x}_t)-\min_{\bm{x}^{*}\in\mathcal{Q}}\sum_{t=1}^Tf_t(\bm{x}^{*})\\
\end{equation*}
\end{small}
is sub-linear in $T$, i.e. $\lim\limits_{T\rightarrow\infty}\frac{1}{T}\mathbf{regret}_T=0$. The sub-linearity implies that when $T$ is large enough, the learner can perform as well as the best fixed decision in hindsight.

Based on the OCO framework, many online learning algorithms have been proposed and successfully applied in various applications
~\cite{li2014online,DeMarzo:2006:OTA:1132516.1132586,Zhao:2013:COA:2487575.2487647}. The two most popular representatives of them are Online Gradient Descent\,(OGD)~\cite{Zinkevich:2003:OCP:3041838.3041955} and Follow-The-Regularized-Leader\,(FTRL)~\cite{ad41159}.

\subsection{Online Conditional Gradient}

Online Conditional Gradient\,(OCG)~\cite{hazan2016introduction} is a projection-free online learning algorithm that eschews the possible computationally expensive projection operation needed in its counterparts, including OGD and FTRL. It enjoys great computational advantage over other online learning algorithms when the decision set is a bounded nuclear norm ball. As will be shown in Section~\ref{Model_and_Algorithm}, we utilize a similar decision set in our proposed convexification scheme, thus we use OCG as the cornerstone in our algorithm design. In the following, we introduce more details about it.

In practice, to ensure that the newly generated decision points lie inside the decision set of interest, most online learning algorithms invoke a projection operation in every iteration. For example, in OGD, when the gradient descent step generates an infeasible iterate that lies out of the decision set, we have to project it back to regain feasibility. In general, this kind of projection amounts to solving a quadratic convex program over the decision set and will not cause much problem. However, when the decision sets are of specific types, such as the bounded nuclear norm ball, the set of all semi-definite matrices and polytopes~\cite{hazan2016introduction, Hazan:2012:POL:3042573.3042808}, it turns to amount to very expensive algebraic operations. To avoid these expensive operations, projection-free online learning algorithm OCG has been proposed. It is much more efficient since it eschews the projection operation by using a linear optimization step instead in every iteration. For example, when the decision set is a bounded nuclear norm ball, the projection operation needed in other online learning algorithms amounts to computing a full singular value decomposition\,(SVD) of a matrix, while the linear optimization step in OCG amounts to only computing the maximal singular vectors, which is at least one order of magnitude simpler. Recently, a decentralized distributed variant of OCG algorithm has been proposed in \cite{pmlr-v70-zhang17g}, which allows for high efficiency in handling large-scale machine learning problems.

\begin{small}
\begin{algorithm}[!t]
\renewcommand{\algorithmicrequire}{\textbf{Input:}}
\renewcommand{\algorithmicensure}{\textbf{Output:}}
\caption{\textsc{Online Conditional Gradient (\textbf{OCG})}}
\label{OCG}
\begin{flushleft}
\textbf{Input:}
Convex set $\mathcal{Q}$, Maximum round number $T$, parameters $\eta$ and $\{\gamma_t\}^1$\protect\footnotemark[5]

\end{flushleft}

\begin{algorithmic}[1]
\STATE Initialize $\bm{x}_{1}\in\mathcal{Q}$
\FOR {$t=1,2,\ldots,T$}
\STATE Play $\bm{x}_t$ and observe $f_t$
\STATE Let $F_t(\bm{x})=\eta\sum_{\tau=1}^{t-1}\langle\nabla f_\tau(\bm{C}_{\tau}),\bm{x}\rangle+\|\bm{x}-\bm{x}_1\|_2^2$
\STATE Compute $\bm{v}_t=\mathrm{argmin}_{\bm{x}\in \mathcal{Q}}\langle\bm{x},\nabla F_t(\bm{x}_{t})\rangle$
\STATE Set $\bm{x}_{t+1}=(1-\gamma_t)\bm{x}_t+\gamma_t\bm{v}_t$
\ENDFOR
\end{algorithmic}
\end{algorithm}
\end{small}
\footnotetext{$ ^1$Usually $\gamma_t$is set to $\frac{1}{t^{1/2}}$}

\section{Online Compact Convexified Factorization Machine}
\label{Model_and_Algorithm}

In this section, we turn to the development of our Online Compact Convexified Factorization Machine.
As introduced before, OCO is the paramount framework for designing online learning algorithms. There are two fundamental requirements in it:
({\romannumeral1}) any instance of all the model parameters should be represented as a single point from a convex compact decision set;
({\romannumeral2}) the loss incurred by the prediction should be formulated as a convex function over the decision set.

Vanilla FM can not directly fit into the OCO framework due to the following two reasons. First, vanilla FM contains both the unconstrained feature weight vector $\bm{\omega}$\,($\omega_0$ can be written into $\bm{\omega}$) and the factorized feature interaction matrix $\bm{V}$; they are separated components of parameters which can not be formulated as a single point from a decision set. Second, as the prediction $\hat{y}$ in vanilla FM is non-convex with respect to $\bm{V}$~\cite{blondel2015convex,yamada2015convex}, when we plug it into the loss function $f_t$ at each round $t$, the resulting loss incurred by the prediction is also non-convex with respect to $\bm{V}$. Although some previous studies have proposed two formulations for convex FM, they cannot fit into the OCO framework either. Similar to vanilla FM, the two convex formulations treat the unconstrained feature weight vector $\bm{\omega}$ and the feature interaction matrix $\bm{Z}$ as separated components of parameters, which violates requirement ({\romannumeral1}).

In order to meet the requirements of the OCO framework, we first invent a new convexification scheme for FM which results in a formulation named as Compact Convexified FM\,(CCFM). Then based on the new formulation, we design an online learning algorithm -- Online Compact Convexified Factorization Machine\,(OCCFM), which is a new variant of the OCG algorithm tailored to our setting.

\subsection{Compact Convexified Factorization Machine}
\label{Convexfication_Scheme}

The main idea of our proposed convexification scheme is to put all the parameters, including the bias term, the first-order feature weight vector and the second-order feature interaction matrix, into a single augmented matrix which is then enforced to be low-rank. As the low-rank property is not a convex constraint, we approximate it with a bounded nuclear norm constraint, which is a common practice in the machine learning community~\cite{blondel2015convex, Shalev-Shwartz:2012:OLO:2185819.2185820, Hazan:2012:POL:3042573.3042808}. The details of the scheme are given in the following.

Recall that in vanilla FM, $\omega_0$, $\bm{\omega}$ and $\bm{V}$ refer to the bias term, the linear feature weight vector and the factorized feature interaction matrix respectively. Due to the non-convexity of the prediction $\hat{y}$ with respect to $\bm{V}$, we adopt a symmetric matrix $\bm{Z}\in\mathbb{R}^{d\times d}$ to model the pairwise feature interaction, which is a popular practice in designing convex variants of FM\,\cite{blondel2015convex, Yamada:2017:CFM:3097983.3098103}.
Then we rewrite the bias term $\omega_0$, the first-order feature weight vector $\bm{\omega}$ and the second-order feature interaction matrix $\bm{Z}$ into a single compact augmented matrix $\bm{C}$:
\begin{small}
\begin{equation*}
\bm{C}=\left[
\begin{array}{ccc}
\bm{Z} & \bm{\omega} \\
\bm{\omega}^T & 2\omega_0
\end{array}
\right],
\end{equation*}
\end{small}
\noindent where $\bm{Z}=\bm{Z}^T\in\mathbb{R}^{d\times d},~\bm{\omega}\in\mathbb{R}^{d},~\omega_0\in\mathbb{R}$.

In order to achieve a model with low complexity, we restrict the augmented matrix $\bm{C}$ to be low-rank.
However, the constraint over the rank of a matrix is non-convex and does not fit into the OCO framework. A typical approximation of the rank of a matrix $\bm{C}$ is its nuclear norm $\|\bm{C}\|_{tr}$ \cite{boyd2004convex}:
$\|\bm{C}\|_{tr}=tr(\sqrt{\bm{C^TC}})=\sum_{i=1}^d\sigma_{i}$,
where $\sigma_i$ is the $i$-th singular value of $\bm{C}$. As the singular values are non-negative, the nuclear norm is essentially a convex surrogate of the rank of the matrix \cite{blondel2015convex, Shalev-Shwartz:2012:OLO:2185819.2185820, Hazan:2012:POL:3042573.3042808}. Therefore it is standard to consider a relaxation that replaces the rank constraint by the bounded nuclear norm $\|\bm{C}\|_{tr}\leq\delta.$

Denote $\mathcal{S}$ as the set of symmetric matrices: $\mathcal{S}^{d\times d}=\{\bm{X}|~\bm{X}\in \mathbb{R}^{d\times d},~\bm{X}=\bm{X}^T\}$, the resulting decision set of the augmented matrices $\bm{C}$ can be written as the following formulation:
\begin{small}
\begin{equation*}
\mathcal{K}=\{\bm{C}|\bm{C}=\left[
\begin{array}{ccc}
\bm{Z} & \bm{\omega} \\
\bm{\omega}^T & 2\omega_0
\end{array}
\right],~\|\bm{C}\|_{tr}\leq\delta,
\bm{Z}\in\mathcal{S}^{d\times d},~\bm{\omega}\in\mathbb{R}^{d},~\omega_0\in\mathbb{R}\}.
\end{equation*}
\end{small}As the set $\mathcal{K}$ is bounded and closed, it is also compact.
Next we prove that the decision set $\mathcal{K}$ is convex in Lemma \ref{Nuclear_Norm_Ball}.

\begin{speciallma}
\label{Nuclear_Norm_Ball}
The set
\begin{small}
$\mathcal{K}=\{\bm{C}|\bm{C}=\left[
\begin{array}{ccc}
\bm{Z} & \bm{\omega} \\
\bm{\omega}^T & 2\omega_0
\end{array}
\right],~\|\bm{C}\|_{tr}\leq\delta,~\bm{Z}\in\mathcal{S}^{d\times d},~\bm{\omega}\in\mathbb{R}^{d},~\omega_0\in\mathbb{R}\}$
\end{small}
is convex.
\end{speciallma}

\begin{proof}
The proof is based on an important property of convex sets: if two sets $S_1$ and $S_2$ are convex, then their intersection $S=S_1\cap S_2$ is also convex \cite{boyd2004convex}. In our case, the decision set $\mathcal{K}$ is an intersection of $\tilde{\mathcal{K}}$ and $\mathcal{B}$, where
\begin{small}
$\tilde{\mathcal{K}}=\{\bm{C}|\bm{C}=\left[
\begin{array}{ccc}
\bm{Z} & \bm{\omega} \\
\bm{\omega}^T & 2\omega_0
\end{array}
\right],~\bm{Z}\in\mathcal{S}^{d\times d},~\bm{\omega}\in\mathbb{R}^{d},~\omega_0\in\mathbb{R}\}$
\end{small}
and
$\mathcal{B}=\{\bm{C}|\|\bm{C}\|_{tr}\leq\delta,~\bm{C}\in \mathbb{R}^{(d+1)\times(d+1)}\}$. To prove $\mathcal{K}$ is convex, we prove that both $\tilde{\mathcal{K}}$ and $\mathcal{B}$ are convex sets.

First, we prove that $\tilde{\mathcal{K}}$ is a convex set. For any two points $\bm{C}_1,\,\bm{C}_2\in\tilde{\mathcal{K}}$ and any $\alpha\in[0,1]$, we have
\begin{small}
\begin{equation*}
\begin{aligned}
\alpha\bm{C}_1+(1-\alpha)\bm{C}_2
=~&\alpha\left[
\begin{array}{ccc}
\bm{Z}_1 & \bm{\omega}_1 \\
\bm{\omega}_1^T & 2\omega_0
\end{array}
\right]
+(1-\alpha)\left[
\begin{array}{ccc}
\bm{Z}_2 & \bm{\omega}_2 \\
\bm{\omega}_2^T & 2\omega_0
\end{array}
\right]\\
=~&\left[
\begin{array}{ccc}
\bm{Z}_{\alpha} & \bm{\omega}_{\alpha}  \\
\bm{\omega}_{\alpha}^T & 2\omega_0
\end{array}
\right],
\end{aligned}
\end{equation*}
\end{small}
where $\bm{Z}_{\alpha}=\alpha\bm{Z}_1+(1-\alpha)\bm{Z}_2$ and $\bm{\omega}_{\alpha}=\alpha\bm{\omega}_1+(1-\alpha)\bm{\omega}_2$.
By definition of $\tilde{\mathcal{K}}$, $\bm{Z}_1=\bm{Z}_1^T$, $\bm{Z}_2=\bm{Z}_2^T$, thus $\alpha\bm{Z}_1+(1-\alpha)\bm{Z}_2=\alpha\bm{Z}_1^T+(1-\alpha)\bm{Z}_2^T$. Following the property of transpose, $\alpha\bm{\omega}_1^T+(1-\alpha)\bm{\omega}_2^T = (\alpha\bm{\omega}_1+(1-\alpha)\bm{\omega}_2)^T$. Therefore we obtain $\alpha\bm{C}_1+(1-\alpha)\bm{C}_2\in\tilde{\mathcal{K}}$. By definition, $\tilde{\mathcal{K}}$ is convex.

Second, it remains to prove that $\mathcal{B}$ is also a convex set. The bounded nuclear norm ball is a typical convex set of matrices \cite{boyd2004convex}, which is widely adopted in the machine learning community \cite{Shalev-Shwartz:2012:OLO:2185819.2185820, hazan2016introduction}. The detailed proof can be found in \cite{boyd2004convex}.

Following the convexity preserving property under the intersection of convex sets, we have $\mathcal{K}=\tilde{\mathcal{K}}\cap\mathcal{B}$ is also convex.
\end{proof}

As the decision set $\mathcal{K}$ consists of augmented matrices with bounded nuclear norm, by the properties of block matrix, we show that the decision set $\mathcal{K}$ is equivalent to a symmetric nuclear norm ball in Lemma \ref{Lemma::Symmetric_NNB}.
\begin{speciallma}
\label{Lemma::Symmetric_NNB}
The sets
$\mathcal{K}=\{\bm{C}|\bm{C}=\left[
\begin{array}{ccc}
\bm{Z} & \bm{\omega} \\
\bm{\omega}^T & 2\omega_0
\end{array}
\right],~\|\bm{C}\|_{tr}\leq\delta,~\bm{Z}\in\mathcal{S}^{d\times d},~\bm{\omega}\in\mathbb{R}^{d}\}$ and $\mathcal{K}^{'}=\{\bm{C}|\bm{C}\in\mathcal{S}^{(d+1)\times(d+1)},\|\bm{C}\|_{tr}\leq\delta\}$ are equivalent: $\mathcal{K}=\mathcal{K}^{'}$.
\end{speciallma}
The proof is straight-forward by writing an arbitrary symmetric matrix into a block form, the details are omitted here.
%

By choosing a point $\bm{C}$ from the compact decision set $\mathcal{K}$, the prediction $\hat{y}(\bm{C})$ of an instance $\bm{x}$ is formulated as:
\begin{small}
\begin{equation*}
\hat{y}(\bm{C})=\frac{1}{2}\bm{\hat{x}}^T\bm{C\hat{x}},
\end{equation*}
\end{small}
where
$\bm{C}\in\mathcal{K},~\bm{\hat{x}}=
\left[
\begin{array}{ccc}
\bm{x} \\
1
\end{array}
\right].
$
Although the prediction function $\hat{y}$ has a similar form with the feature interaction component $\frac{1}{2}\bm{x}^T\bm{Z}\bm{x}$ of vanilla FM, they are intrinsically different.
Plugging the formulation of $\bm{C}=\left[
\begin{array}{ccc}
\bm{Z} & \bm{\omega} \\
\bm{\omega}^T & 2\omega_0
\end{array}
\right]$ and $\bm{\hat{x}}=
\left[
\begin{array}{ccc}
\bm{x} \\
1
\end{array}
\right]$ into the prediction function, we have
\begin{small}
\begin{equation*}
\hat{y}(\bm{C})=\frac{1}{2}\bm{\hat{x}}^T\bm{C\hat{x}}=\omega_0+\bm{\omega}^T\bm{x}+\frac{1}{2}\bm{x}^T\bm{Z}\bm{x}.
\end{equation*}
\end{small}
Therefore the prediction function $\hat{y}(\bm{C})$ contains all the three components in vanilla FM, including the global bias, the first-order feature weight and the second-order feature interactions.
Most importantly, $\hat{y}(\bm{C})$ is a convex function in $\bm{C}$, as shown in Lemma. \ref{Lemma::Convex_Function}.

\begin{speciallma}
\label{Lemma::Convex_Function}
The prediction function $\hat{y}(\bm{C})$ is a convex function of $\bm{C}\in\mathcal{K}$, where $\mathcal{K}=\{\bm{C}|\bm{C}=\left[
\begin{array}{ccc}
\bm{Z} & \bm{\omega} \\
\bm{\omega}^T & 2\omega_0
\end{array}
\right],~\|\bm{C}\|_{tr}\leq\delta,~\bm{Z}\in\mathcal{S}^{d\times d},~\bm{\omega}\in\mathbb{R}^{d},~\omega_0\in\mathbb{R}\}$.
\end{speciallma}
\begin{proof}
By definition, $\hat{y}(\bm{C})$ is a separable function:
\begin{small}
\begin{equation*}
\hat{y}(\bm{C})=g(\bm{Z})+h(\bm{\omega})+\omega_0,
\end{equation*}
\end{small}
\noindent where $g(\bm{Z})=\frac{1}{2}\bm{x}^T\bm{Z}\bm{x},~h(\bm{\omega})=\bm{\omega}^{T}\bm{x}$.
By definition, $g(\bm{Z})$ and $h(\bm{\omega})$ are linear functions with respect with $\bm{Z}$ and $\bm{\omega}$ correspondingly, and thus are convex functions.

Consider $\forall\,\alpha\in[0,\,1],~\forall\,\bm{C}_1=\left[
\begin{array}{ccc}
\bm{Z}_1 & \bm{\omega}_1 \\
\bm{\omega}_1^T & 2\omega_0
\end{array}
\right],~\bm{C}_2=\left[
\begin{array}{ccc}
\bm{Z}_2 & \bm{\omega}_2 \\
\bm{\omega}_2^T & 2\omega_0
\end{array}
\right]\in\mathcal{K}$,
according to the separable formulation of $\hat{y}$, we obtain
\begin{small}
\begin{equation*}
\hat{y}(\alpha\bm{C}_1+(1-\alpha)\bm{C}_2)
=~g(\alpha\bm{Z}_1+(1-\alpha)\bm{Z}_2)+h(\alpha\bm{\omega}_1+(1-\alpha)\bm{\omega}_2)+\omega_0.
\end{equation*}
\end{small}
By definition of the linear functions $g(\bm{Z})$ and $h(\bm{\omega})$ and the rearrangement of the formulation, we have
\begin{small}
\begin{equation*}
\begin{aligned}
&g(\alpha\bm{Z}_1+(1-\alpha)\bm{Z}_2)+h(\alpha\bm{\omega}_1+(1-\alpha)\bm{\omega}_2)+\omega_0\\
=~&\alpha g(\bm{Z}_1)+(1-\alpha)g(\bm{Z}_2)+\alpha h(\bm{\omega}_1)+(1-\alpha)h(\bm{\omega}_2)+\omega_0\\
=~&\alpha(g(\bm{Z}_1)+h(\bm{\omega}_1)+\omega_0)+(1-\alpha)(g(\bm{Z}_2)+h(\bm{\omega}_2)+\omega_0).
\end{aligned}
\end{equation*}
\end{small}
By applying the separable formulation of $\hat{y}$ again, we obtain
\begin{small}
\begin{equation*}
\alpha(g(\bm{Z}_1)+h(\bm{\omega}_1)+\omega_0)+(1-\alpha)(g(\bm{Z}_2)+h(\bm{\omega}_2)+\omega_0)
=~\alpha \hat{y}(\bm{C}_1)+(1-\alpha)\hat{y}(\bm{C}_2).
\end{equation*}
\end{small}
Therefore, $\hat{y}(\alpha\bm{C}_1+(1-\alpha)\bm{C}_2)=\hat{y}(\alpha\bm{C}_1+(1-\alpha)\bm{C}_2)$, by definition, $\hat{y}(\bm{C})$ is a convex linear function with respect with $\bm{C}\in\mathcal{K}$.
\end{proof}

Next, we show that given the above prediction function $\hat{y}(\bm{C}):\mathcal{K}\rightarrow\mathbb{R}$ and any convex loss function $f(y):\mathbb{R}\rightarrow\mathbb{R}$, the nested function $f(\hat{y}(\bm{C}))=f(\bm{C})$ is also convex, as shown in Lemma \ref{Lemma::Nested_Function}:
\begin{speciallma}
\label{Lemma::Nested_Function}
Let $\hat{y}(\bm{C})=\frac{1}{2}\bm{\hat{x}}^T\bm{C\hat{x}},~\bm{C}\in\mathcal{K},~\bm{\hat{x}}^T=[\bm{x}^T,1],~\bm{x}\in\mathbb{R}^d$, and $f(y):\mathbb{R}\rightarrow\mathbb{R}$ be arbitrary convex function, the nested function $f(\hat{y}(\bm{C}))=f(\bm{C}):\mathcal{K}\rightarrow\mathbb{R}$ is also convex with respect to $\bm{C}\in\mathcal{K}$.
\end{speciallma}
\begin{proof}
Let $\bm{C}_1\in\mathcal{K}$ and $\bm{C}_2\in\mathcal{K}$ be any two points in the set $\mathcal{K}$, and $f(y)$ be an arbitrary convex function, by definition of $\hat{y}$, $\forall\,\alpha\in[0,1]$, we obtain
\begin{small}
\begin{equation*}
f(\hat{y}(\alpha\bm{C}_1+(1-\alpha)\bm{C}_2))=
f(\alpha\hat{y}(\bm{C}_1)+(1-\alpha)\hat{y}(\bm{C}_2)).
\end{equation*}
\end{small}
By the convexity of $f(y)$, we obtain
\begin{small}
\begin{equation*}
f(\alpha\hat{y}(\bm{C}_1)+(1-\alpha)\hat{y}(\bm{C}_2))\leq\alpha f(\hat{y}(\bm{C}_1))+(1-\alpha)f(\hat{y}(\bm{C}_2)).
\end{equation*}
\end{small}
Therefore $f(\alpha\bm{C}_1+(1-\alpha)\bm{C}_2)\leq\alpha f(\bm{C}_1)+(1-\alpha)f(\bm{C}_2)$.
By definition, the nested function $f(\hat{y}(\bm{C}))=f(\bm{C})$ is also convex.
\end{proof}

In summary, by introducing the new convexification scheme, we obtain a new formulation for FM, in which the decision set is convex compact and the nested loss function is convex. We refer to the resulting formulation as \textbf{Compact Convexified FM\,(CCFM)}.

The comparison between vanilla FM, convex FM and the proposed CCFM is illustrated in Table \ref{Characteristics}.


\noindent\textbf{Comparison between vanilla FM and CCFM.}\quad
The differences between vanilla FM and CCFM are three-folded:
({\romannumeral1}) the primal difference between vanilla FM and CCFM is the convexity of prediction function: the prediction function of vanilla FM is non-convex while it is convex in CCFM.
({\romannumeral2}) vanilla FM factorizes the feature interaction matrix $\bm{Z}$ into $\bm{VV}^T$, thus restricting $\bm{Z}$ to be symmetric and positive semi-definite; while in CCFM, there is no specific restriction on $\bm{Z}$ except symmetry.
({\romannumeral3}) vanilla FM only considers the interactions between distinct features: $x_ix_j,~\forall\,i,\,j\in\{1,\ldots,d\},~i\neq j$, while CCFM models the interactions between all possible feature pairs: $x_ix_j,~\forall\,i,\,j\in\{1,\ldots,d\}$. By rewriting the prediction function as $\hat{y}=\langle\bm{C},\bm{\hat{x}}\bm{\hat{x}^T}-diag(\bm{\hat{x}}\circ\bm{\hat{x}})\rangle$ where $\circ$ is the element-wise product, we can easily leave out the interactions between same features without changing the theoretical guarantees in CCFM.
Combining ({\romannumeral2}) and ({\romannumeral3}), we find that CCFM allows for general modeling of feature interactions, which improves its expressiveness.

\begin{small}
\begin{table}[!t]
\tabcolsep=2.0pt
\vspace{-2pt}
\centering
\caption{ Comparison between different FM formulations\protect\footnotemark[1]$ ^2$}\label{Characteristics}
\scalebox{1.0}
{\begin{tabular}{|c|p{1.4cm}<{\centering}|p{1.4cm}<{\centering}|p{1.4cm}<{\centering}|p{1.4cm}<{\centering}|} \hline
\textbf{Properties}&\textbf{Convex}&\textbf{Compact}&\textbf{All-Pairs}&\textbf{Online}\\ \hline\hline
\textbf{FM} & $\backslash$ & $\backslash$ & $\backslash$ & $\backslash$ \\ 
\textbf{CFM \cite{blondel2015convex}} & $\surd$ & $\backslash$ & $\surd$ & $\backslash$ \\ 
\textbf{CFM \cite{Yamada:2017:CFM:3097983.3098103}} & $\surd$  & $\backslash$ & $\backslash$ & $\backslash$ \\ 
\textbf{CCFM} & $\surd$ & $\surd$  & $\surd$ & $\surd$ \\
\hline\end{tabular}}
\end{table}
\end{small}
\footnotetext[1]{$ ^2$ In Table \ref{Characteristics}, FM, CFM, CCFM refer to vanilla FM , Convex FM and Compact Convex FM respectively. The term "Convex" indicates whether the prediction function is convex; the term "Compact" indicates whether the feasible set of the formulation is compact; the term "All-Pairs" indicates whether all the pair-wise feature interactions are involved in the formulation; the term "Online" indicates whether the formulation fits the OCO framework for Online learning.}
\noindent\textbf{Comparison between convex FM and CCFM.}\quad
Although convex FM involves convex prediction functions, they are still inherently different from CCFM: ({\romannumeral1}) in convex FM, the first-order feature weight vector and second-order feature interaction matrix are separated, resulting in a non-compact formulation. As pointed out in \cite{blondel2015convex, yamada2015convex}, the separated formulation makes it difficult to jointly learn the two components in the training process, which can cause inconsistent convergence rates; But in CCFM, all the parameters are written into a compact augmented matrix. ({\romannumeral2}) in CCFM, we restrict that the compact augmented matrix $\bm{C}$ is low-rank; while convex FM formulations only require the second-order matrix $\bm{Z}$ to be low-rank, leaving the first-order feature weight vector $\bm{\omega}$ unbounded.  ({\romannumeral3}) convex FM can not fit into the OCO framework easily due to its non-compact decision set while CCFM seamlessly fits the OCO framework.




\subsection{Online Learning Algorithm for Compact Convexified Factorization Machine}
With the convexification scheme mentioned above, we have got a convex compact set and a convex loss function in the new formulation of Compact Convexified Factorization Machine (CCFM), which allows us to design the online learning algorithm following the OCO framework \cite{Shalev-Shwartz:2012:OLO:2185819.2185820}.
As shown in the Preliminaries, the decision set is an important issue that affects the computational complexity of the projection step onto the set.
In CCFM, the decision set is a subset of the bounded nuclear norm ball, where the projection step involves singular value decomposition and takes super-linear time via our best known methods \cite{hazan2016introduction}. Although Online Gradient Descent (OGD) and Follow-The-Regularized-Leader (FTRL) are the two most classical online learning  algorithms, they suffer from the expensive SVD operation in the projection step, thus do not work well on this specific decision set.

Meanwhile, the nuclear norm ball is a typical decision set where the expensive projection step can be replaced by the efficient linear optimization subroutine in the projection-free OCG algorithm. As the decision set of CCFM is a subset of the bounded nuclear norm ball, we first propose the online learning algorithm for CCFM (OCCFM), which is essentially an OCG variant. Then we provide the theoretical analysis for OCCFM in details.



\subsubsection{Algorithm}
As introduced in the Preliminaries, it is a typical practice to apply OCG algorithm on the bounded nuclear norm ball in the OCO framework. This typical example is related to CCFM, but is also inherently different. On one hand, the decision set of CCFM is a subset of the bounded nuclear norm ball. Thus it is tempting to apply the subroutine of OCG over the bounded nuclear norm ball on the decision set of CCFM; on the other hand, the decision set of CCFM is also a subset of symmetric matrices, and it remains theoretically unknown whether the subroutine can be applied to the symmetric bounded nuclear norm ball. Therefore, it is a non-trivial problem to design an OCG-alike algorithm for the CCFM due to its specific decision set.


To avoid the expensive projection onto the decision set of CCFM, we follow the core idea of OCG to replace the projection step with a linear optimization problem over the decision set. What remains to be done is to design an efficient subroutine that solves the linear optimization over this set in low computational complexity. Recall the procedure of OCG in Algorithm \ref{OCG}, the validity of this subroutine depends on the following two important requirements:
\begin{itemize}
  \item First, the subroutine should solve the linear optimization over the decision set with low computational complexity; in our case, the subroutine should generate the optimal solution $\bm{\hat{C}}_t$ to the problem $\bm{\hat{C}}_t=\mathrm{argmin}_{\bm{C}\in \mathcal{K}}\langle\bm{C},\nabla F_t(\bm{C}_{t})\rangle$ with linear or lower computational complexity, where $\bm{C}_t$ is the iterate of $\bm{C}$ at round $t$, $\mathcal{K}$ is the decision set of CCFM, $\nabla F_t(\bm{C}_{t})$ is the sum of gradients of the loss function incurred until round $t$.
  \item Second, the subroutine should be closed over the convex decision set; in our case, the augmented matrix $\bm{C}_t$ should be inside the decision set $\mathcal{K}$ throughout the algorithm: $\bm{C}_{t+1}=(1-\gamma_{t})\bm{C}_{t}+\gamma_{t}\bm{\hat{C}}_{t}\in\mathcal{K},~\forall\,t=1,\ldots,T$, where $\bm{\hat{C}}_t$ is the output of the subroutine and $\gamma_t$ is the step-size in round $t$.
\end{itemize}

\begin{small}
\begin{algorithm}[!t]
\caption{\textsc{Online Compact Convexified Factorization Machine (\textbf{OCCFM})}}
\label{FWA}
\begin{flushleft}
\textbf{Input:} Convex set $\mathcal{K}=\{\bm{C}|~\|\bm{C}\|_{tr}\leq\delta,~\bm{C}\in\mathcal{S}^{(d+1)\times(d+1)}\}$, Maximum\\\quad\quad\quad round number $T$, parameters $\eta$ and $\{\gamma_t\}^3$\protect\footnotemark[6].
\end{flushleft}

\begin{algorithmic}[1]
\STATE Initialize $\bm{C}_{1}\in\mathcal{K}$
\FOR {$t=1,2,\ldots,T$}
\STATE Get input $(\bm{x}_t,y_t)$ and compute $f_t(\bm{C}_t)$
\STATE Let $F_t(\bm{C})=\eta\sum_{\tau=1}^{t-1}\langle\nabla f_\tau(\bm{C}_{\tau}),~\bm{C}\rangle+\|\bm{C}-\bm{C}_1\|_2^2$
\STATE Compute the maximal singular vectors of $-\nabla F_t(\bm{C}_{t}):$ $\bm{u}$ and $\bm{v}$
\STATE Solve $\bm{\hat{C}}_t=\mathrm{argmin}_{\bm{C}\in \mathcal{K}}\langle\bm{C},\nabla F_t(\bm{C}_{t})\rangle$ with $\bm{\hat{C}}_t=\delta\bm{u}\bm{v}^T$
\STATE Set $\bm{C}_{t+1}=(1-\gamma_t)\bm{C}_t+\gamma_t\bm{\hat{C}}_t$
\ENDFOR
\end{algorithmic}
\end{algorithm}
\end{small}
\footnotetext{$ ^3$Similarly,  $\gamma_t$is set to $\frac{1}{t^{1/2}}$}


Considering these two requirements, we propose a subroutine of the linear optimization over the symmetric bounded nuclear norm ball, based on which we build the online learning algorithm for the Compact Convexified Factorization Machine. Specifically, we prove that in each round, the linear optimization over the decision set in CCFM is also equivalent to the computation of maximal singular vectors of a specific symmetric matrix. This subroutine can be solved in linear time via the power method, which validates its efficiency.

The procedure is detailed in Algorithm \ref{FWA}  where $\bm{u}$ and $\bm{v}$ are the left and right maximal singular vectors of $-\nabla F_t(\bm{C}_t)$ respectively. The projection step is replaced with the subroutine at line $5-6$ in the algorithm and the detailed analysis of the algorithm will be presented later.

\subsubsection{Theoretical Analysis}
The Online Compact Convexified Factorization Machine (OCCFM) algorithm is essentially a variant of OCG algorithm, which preserves the similar process. First, we show that OCCFM  satisfies the aforementioned two requirements, then we prove the regret bound of OCCFM following the OCO framework.


To prove that OCCFM satisfies the aforementioned two requirements, we present the theoretical guarantee in Theorem \ref{Linear_OPT} and Theorem \ref{Decision_Set} respectively.

\begin{specialthm}
\label{Linear_OPT}
In Algorithm \ref{FWA}, the subroutine of linear optimization amounts to the computation of maximal singular vectors:
Given $\bm{C}_t\in\mathcal{K}=\{\bm{C}|~\|\bm{C}\|_{tr}\leq\delta,~\bm{C}\in\mathcal{S}^{(d+1)\times(d+1)}\}$, $\bm{\hat{C}}_t=\mathop{\mathrm{argmin}}\limits_{\bm{C}\in \mathcal{K}}\langle\bm{C},\nabla F_t(\bm{C}_{t})\rangle=\delta\bm{u}_1\bm{v}_1^T$, where $\bm{u}_1$ and $\bm{v}_1$ are the maximal singular vectors of $-\nabla F_t(\bm{C}_{t})$.
\end{specialthm}

\begin{specialthm}
\label{Decision_Set}
In Algorithm \ref{FWA}, the subroutine is closed over the decision set $\mathcal{K}$, i.e. the augmented matrix  $\bm{C}_t$ generated at each iteration $t$ is inside the decision set $\mathcal{K}$: $\bm{C}_t\in\mathcal{K}$, $\forall\,t=1,\ldots,T$, where $\mathcal{K}=\{\bm{C}|\|\bm{C}\|_{tr}\leq\delta,~\bm{C}\in\mathcal{S}^{(d+1)\times(d+1)}\}$.
\end{specialthm}

Before we proceed with the proof of Theorem \ref{Linear_OPT} and Theorem \ref{Decision_Set}, we provide a brief introduction to the Singular Value Decomposition (SVD) of a matrix, which is frequently used in the proofs of the theorems.
For any matrix $\bm{C}\in \mathbb{R}^{m\times n}$, the Singular Value Decomposition is a factorization of the form $\bm{C}=\bm{U}\bm{\Sigma}\bm{V}^{T}$, where $\bm{U}\in\mathbb{R}^{m\times K}$ and $\bm{V}\in\mathbb{R}^{n\times K},~(K=\min\{m,\,n\})$ are the orthogonal matrices, and $\bm{\Sigma}\in\mathbb{R}^{K\times K}$ is a diagonal matrix. The diagonal entries $\sigma_i$ of $\bm{\Sigma}$ are non-negative real numbers and known as singular values of $\bm{C}$. Conventionally, the singular values are in a permutation of non-increasing order: $\sigma_1\geq\sigma_2\geq,\ldots,\geq\sigma_K$.

Next, we prove Lemma \ref{Symmetric_Matrix_SVD_EIG}, which is  an important part for the proof of both theorems.
In this lemma, we show that the outer product of the left and right maximal singular vectors of a symmetric matrix (as part of the subroutine in Algorithm \ref{FWA}) is also a symmetric matrix. 

\begin{speciallma}
\label{Symmetric_Matrix_SVD_EIG}
Let $\bm{C}\in \mathcal{S}^{d\times d}$ be an arbitrary real symmetric matrix, and $\bm{C}=\bm{U}\bm{\Sigma}\bm{V}^T$ be the singular value decomposition of $\bm{C}$,
$\bm{u}_1,\,\bm{v}_1\in R^{d}$ be the left and right maximal singular vectors of $\bm{C}$ respectively, then the matrix $\bm{u}_1\bm{v}_1^T$ is also symmetric: $\bm{u}_1\bm{v}_1^T\in\mathcal{S}^{d\times d}$.
\end{speciallma}
\begin{proof}
The proof is based on the connection between the singular value decomposition and eigenvalue decomposition of the real symmetric matrix. Denote the singular value decomposition of $\bm{C}\in\mathcal{S}^{d\times d}$ as: $\bm{C}=\bm{U}\bm{\Sigma}\bm{V}^T=\sum_{i=1}^d\sigma_i\bm{u}_i\bm{v}_i^T$ and its eigenvalue decomposition as: $\bm{C}=\bm{Q}\bm{\Lambda}\bm{Q}^T=\sum_{j=1}^d\lambda_j\bm{q}_j\bm{q}_j^T$.

First, we show that the eigenvalues of $\bm{C}\bm{C}^{T}$ are the square of the singular values of $\bm{C}$:
\begin{small}
\begin{equation*}
\bm{D}=\bm{C}^{T}\bm{C}=\bm{V}\bm{\Sigma}\bm{U}^T\bm{U}\bm{\Sigma}^T\bm{V}^T=\bm{V}\bm{\Sigma}\bm{\Sigma}^T\bm{V}^T.
\end{equation*}
\end{small}
As $\bm{V}$ is an orthogonal matrix, $\bm{V\Sigma}\bm{\Sigma}^T\bm{V}^T$ is an eigenvalue decomposition of $\bm{D}$, where $\bm{\Sigma}\bm{\Sigma}^T$ is the diagonal matrix of eigenvalues of $\bm{D}$, and $\bm{v}_i,\,\forall\,i\in\{1,\ldots,d\}$ is the eigenvector of $\bm{D}$.

Since $\bm{C}$ is symmetric, $\bm{D}=\bm{C}^T\bm{C}=\bm{C}^2$. For any eigenvalue $\lambda$ and eigenvector $\bm{q}$ of $\bm{C}$: $\bm{C}\bm{q}=\lambda\bm{q}$, we have
\begin{small}
\begin{equation*}
\bm{D}\bm{q}=\bm{C}^T\bm{C}\bm{q}=\bm{C}^T(\bm{C}\bm{q})=\lambda\bm{C}^T\bm{q}=\lambda\bm{C}\bm{q}=\lambda^2\bm{q}.
\end{equation*}
\end{small}
Thus the square of every eigenvalue $\lambda$ of $\bm{C}$ is also the eigenvalue $\lambda^2$ of $\bm{D}$ and every eigenvector $\bm{q}$ of $\bm{C}$ is also the eigenvector $\bm{v}$ of $\bm{D}$.

By rearranging the eigenvalues of $\bm{C}$ so that $|\hat{\lambda}_i|,~i=1,\ldots,d$ is non-increasing, we have $\sigma_i=|\hat{\lambda}_i|$. By rewriting the eigenvalue decomposition $\bm{C}$ as $\sum_{i=1}^d\hat{\lambda}_{i}\hat{\bm{q}}_i\hat{\bm{q}}_i^T=\sum_{i=1}^d|\hat{\lambda}_{i}|\hat{\bm{p}}_i\hat{\bm{q}}_i^T=\sum_{i=1}^d\sigma_i\bm{u}_i\bm{v}_i^T$, where $\hat{\bm{p}}_i=sgn(\hat{\lambda}_i)\hat{\bm{q}}_i$, we have $\hat{\bm{p}}_i=\bm{u}_i,~\hat{\bm{q}}_i=\bm{v}_i,~\forall\,i=1,\ldots,d$.

Therefore $\bm{u}_i\bm{v}_i^T=sgn(\lambda_i)\hat{\bm{q}}_i\hat{\bm{q}}_i^T,~\forall\,i=\{1,\ldots,d\}$ is a symmetric matrix.

\end{proof}

With these preparations, we prove Theorem \ref{Linear_OPT}.

\begin{proof}[\textbf{Proof of Theorem \ref{Linear_OPT}}]
First,
recall that for any $\bm{C}\in \mathcal{S}^{(d+1)\times (d+1)}$, the SVD of $\bm{C}$ is
$\bm{C}=\bm{U}\bm{\Sigma}\bm{V}=\sum_{i=1}^{d+1}\sigma_i\bm{u}_i\bm{v}_i^T$.
Following the conclusion in Lemma \ref{Symmetric_Matrix_SVD_EIG}, we have $|\bm{u}_i|=|\bm{v}_i|,~\forall\,i\in\{1,\ldots,d+1\}$, we rewrite the decision set $\mathcal{K}$ with the SVD of $\bm{C}$:

\begin{small}
\begin{equation*}
\begin{aligned}
\mathcal{K}=&\{\bm{C}|\bm{C}=\bm{U}\bm{\Sigma}\bm{V}^T=\sum_{i=1}^{d+1}\sigma_i\bm{u}_i\bm{v}_i^T,~|\bm{u}_i|=|\bm{v}_i|,~\forall\,i,\\
&~\forall\,\bm{C}\in\mathcal{S}^{(d+1)\times(d+1)},~\sum_{i=1}^{d+1}\sigma_{i}\leq\delta,~\bm{U},\,\bm{V}\in\mathbb{R}^{(d+1)\times(d+1)}\}.
\end{aligned}
\end{equation*}
\end{small}
Recall that in Algorithm \ref{FWA}, the points are generated from the decision set $\mathcal{K}:~\bm{C}_t,~\bm{C}_1\in\mathcal{K}$, we have
\begin{small}
\begin{equation*}
\nabla F_t(\bm{C}_t)=\frac{\eta}{2}\sum_{\tau=1}^{t-1}f_{\tau}(\bm{C}_{\tau})\bm{\hat{x}}_t\bm{\hat{x}}_t^T+2(\bm{C}_t-\bm{C}_1)\in\mathcal{S}^{(d+1)\times(d+1)}.
\end{equation*}
\end{small}
Denote $-\nabla F_t(\bm{C}_t)$ as $\bm{H}_t\in\mathcal{S}^{(d+1)\times(d+1)}$ and the SVD of $\bm{H}_t$ as $\bm{H}_t=\sum_{i=1}^{d+1}\theta_i\bm{\mu}_i\bm{\nu}_i^T$, where $\theta_i,~\forall\,i\in\{1,\ldots,d+1\}$ are the singular values of $\bm{H}_t$ and $\bm{\mu}_i,~\bm{\nu}_i$ are the left and right singular vectors respectively, the subroutine in Algorithm \ref{FWA} solves the following linear optimization problem:
\begin{small}
\begin{equation*}
\begin{aligned}
&\min.~&&\langle\bm{C},~\nabla F_t(\bm{C}_t)\rangle\\
&s.t.~&&\bm{C}\in \mathcal{S}^{(d+1)\times (d+1)},~\|\bm{C}\|_{tr}\leq \delta.
\end{aligned}
\end{equation*}
\end{small}
The objective function can be rewritten as $\langle\bm{C},\nabla F_t(\bm{C}_{t})\rangle=\langle\bm{C},-\bm{H}_t\rangle$.
and the linear optimization problem can be rewritten as:
\begin{small}
\begin{equation*}
\mathop{\mathrm{argmin}}\limits_{\bm{C}\in \mathcal{K}}\langle\bm{C},\nabla F_t(\bm{C}_{t})\rangle
=\mathop{\mathrm{argmax}}\limits_{\bm{C}\in\mathcal{K}}\langle\sum_{i=1}^{d+1}\sigma_i\bm{u}_i\bm{v}_i^T,~ \bm{H}_t\rangle.
\end{equation*}
\end{small}
Using the invariance of trace of scalar, we have
\begin{small}
\begin{equation*}
\begin{aligned}
\mathop{\mathrm{argmax}}\limits_{\bm{C}\in\mathcal{K}}\langle\sum_{i=1}^{d+1}\sigma_i\bm{u}_i\bm{v}_i^T,~\bm{H}_t\rangle
=&\mathop{\mathrm{argmax}}\limits_{\bm{C}\in\mathcal{K}}\sum_{i=1}^{d+1}\sigma_i\bm{u}_i^T\bm{H}_t\bm{v}_{i}\\
\leq&\mathop{\mathrm{argmax}}\limits_{\bm{C}\in\mathcal{K}}\sum_{i=1}^{d+1}\sigma_i\xi_i\theta_{\pi(i)},~\forall\,\xi_i\in\{-1,1\}.
\end{aligned}
\end{equation*}
\end{small}
where $\theta_{\pi(i)},\forall\,i$ is a permutation of the singular values of matrix $\bm{H}_t$, and $\pi$ is the arbitrary permutation over $[d+1]$. The last inequality can be attained following Lemma \ref{proof_lemma}.

Based on the non-negativity of singular values and the rearrangement inequality, we have:
\begin{small}
\begin{equation*}
\sum_{i=1}^{d+1}\sigma_i\xi_i\theta_{\pi(i)}\leq~\sum_{i=1}^{d+1}\sigma_i\theta_{\pi(i)}\leq~\sum_{i=1}^{d+1}\sigma_i\theta_i
\leq~\sum_{i=1}^{d+1}\sigma_i\cdot\theta_1=~\|\bm{C}\|_{tr}\theta_1.
\end{equation*}
\end{small}
where $\theta_1$ is the maximal singular value of $\bm{H}_t=\nabla F_t(\bm{C}_t)$.

Recall that $\bm{H}_t=-\nabla F_t(\bm{C}_t)\in\mathcal{S}^{(d+1)\times (d+1)}$, following Lemma \ref{Symmetric_Matrix_SVD_EIG}, we have $\bm{\mu}_i\bm{\nu}_i^T\in\mathcal{S}^{(d+1)\times (d+1)},~\forall\,i$. Therefore the equality can be attained when $\bm{u}_i=\bm{\mu}_i,~\bm{v}_i=\bm{\nu}_i,~\forall\,i\in\{1,\ldots,d+1\}$, and $\sigma_1=\delta,~\sigma_i=0,~\forall\,i=2,\cdots,d+1$.

In this case,  $\hat{\bm{C}}_t=\sum_{i=1}^{d+1}\bm{u}_i\sigma_i\bm{v}_i^T=\delta\bm{\mu}_1\bm{\nu}_1^T$. Following Lemma \ref{Symmetric_Matrix_SVD_EIG}, $\hat{\bm{C}}_t=\delta\bm{\mu}_1\bm{\nu}_1^T\in\mathcal{S}^{(d+1)\times(d+1)}$ and $\|\hat{\bm{C}}_t\|_{tr}=\sum_{i=1}^d\sigma_i=\delta$.

As a result, $\sum_{i=1}^{d+1}\bm{u}_i\sigma_i\bm{v}_i^T\in\mathcal{K}$ and thus we have $\bm{\hat{C}}=\delta\bm{\mu}_1\bm{\nu}_1^T=\mathop{\mathrm{argmin}}\limits_{\bm{C}\in \mathcal{K}}\langle\bm{C},\nabla F_t(\bm{C}_{t})\rangle$, 
which indicates that the linear optimization over the decision set $\mathcal{K}$ amounts to the computation of the maximal singular vectors of the symmetric matrix $\bm{H}_t$.
\end{proof}

\begin{speciallma}
\label{proof_lemma}
Let $\bm{A}\in \mathbb{R}^{n\times n}$ and $g:~\mathbb{R}^n\rightarrow \mathbb{R}$ be a twice-differentiable convex function, and $\sigma_i(\bm{A}),\forall\,i$ be the singular values of matrix $\bm{A}$. For any two orthonormal bases $\{\bm{u}_1,\ldots,\bm{u}_n\}$ and $\{\bm{v}_1,\ldots,\bm{v}_n\}$ of $~\mathbb{R}^n$, there exists a permutation $\pi$ over $[n]$ and $\xi_1,...,\xi_n\in\{-1,~1\}$ such that:
\begin{small}
\begin{equation*}
\begin{aligned}
&g(\bm{u}_1^T\bm{A}\bm{v}_1,\bm{u}_2^T\bm{A}\bm{v}_2,\ldots,\bm{u}_n^T\bm{A}\bm{v}_n)\leq \\ &g(\xi_1\sigma_{\pi(1)}(\bm{A}),\xi_2\sigma_{\pi(2)}(\bm{A}),\ldots,\xi_n\sigma_{\pi(n)}(\bm{A})).
\end{aligned}
\end{equation*}
\end{small}
\end{speciallma}
The proof can be found in \cite{DBLP:journals/corr/abs-1708-02105} and the details are omitted here.


Now we are ready to prove the Theorem \ref{Decision_Set}.
\begin{proof}[\textbf{Proof of Theorem \ref{Decision_Set}}]
The proof is conducted with the mathematical induction method.

When $t=1$, recall the initialization of Algorithm \ref{FWA}, $\bm{C}_1\in\mathcal{K}$, thus the proposition stands when $t=1$.

Now assuming the proposition stands for $t>1$: $\bm{C}_t\in\mathcal{K},~\forall\,t\in \{2,\ldots,T\}$, we prove that $\bm{C}_{T+1}\in\mathcal{K}$.
According to the assumption, the augmented matrix is inside the decision set: $\bm{C}_T\in\mathcal{K}$, thus $\|\bm{C}_T\|_{tr}\leq\delta$ and $\bm{C}_T\in\mathcal{S}^{(d+1)\times(d+1)}$. Following the formulation of CCFM and Algorithm \ref{FWA}, we have
\begin{small}
\begin{equation*}
\nabla F_T(\bm{C}_T)=\frac{\eta}{2}\sum_{t=1}^{T}f_t(\bm{C}_t)\bm{\hat{x}}_t\bm{\hat{x}}_t^T+2(\bm{C}_T-\bm{C}_1)\in\mathcal{S}^{(d+1)\times(d+1)},
\end{equation*}
\end{small}
Based on Lemma \ref{Symmetric_Matrix_SVD_EIG} and Theorem \ref{Linear_OPT}, we have
\begin{small}
\begin{equation*}
\bm{\hat{C}}_T=\mathop{\mathrm{argmin}}\limits_{\bm{C}\in \mathcal{K}}\langle\bm{C},\nabla F_T(\bm{C}_{T})\rangle=\delta\bm{u}_1\bm{v}_1^T\in\mathcal{S}^{(d+1)\times(d+1)},
\end{equation*}
\end{small}
where $\bm{u}_1$ and $\bm{v}_1$ are the maximal singular vectors of $-\nabla F_T(\bm{C}_T)$.
Moreover $\|\bm{\hat{C}}_T\|_{tr}=\|\delta\bm{u}_1\bm{v}^T_1\|_{tr}=\delta$. Therefore $\bm{\hat{C}}_T\in\mathcal{K}$.
By the convexity of decision set $\mathcal{K}$, we obtain
\begin{small}
\begin{equation*}
\bm{C}_{T+1}=\gamma\bm{C}_T+(1-\gamma)\bm{\hat{C}_T}\in\mathcal{K}.
\end{equation*}
\end{small}
Therefore the induction stands when $t=T+1$.

In summary, the induction stands for both $t=1$ and $t\in{2,\ldots,T}$, $\forall\,T>1$, which indicates $\bm{C}_t\in\mathcal{K},~\forall\,t\in\mathbb{Z}^{+}$.

\end{proof}


With Theorem \ref{Linear_OPT} and Theorem \ref{Decision_Set}, we prove that the aforementioned two requirements are satisfied. Thus the subroutine in Online Compact Convexified Factorization Machine is a valid conditional gradient step in OCG algorithm, making OCCFM a valid OCG variant.
Following the theoretical analysis of OCG in the OCO framework, we prove that the regret of Algorithm \ref{FWA} after $T$ rounds is sub-linear in $T$, as shown in Theorem \ref{Bound}.

\begin{specialthm}
\label{Bound}
The Online Compact Convexified Factorization Machine with parameters $\eta=\frac{D}{4GT^{3/4}}$, $\gamma_t=\frac{1}{t^{1/2}}$, attains the following guarantee$ ^4$\footnote{$^4$We have reivised the minor mistakes made in the original proof and thus give a slightly different bound here.}:
\begin{small}
\begin{equation*}
regret_T=\sum_{t=1}^Tf_t(\bm{C}_t)-\min_{\bm{C}^{*}\in\mathcal{K}}\sum_{t=1}^Tf_t(\bm{C}^{*})\leq 12DGT^{3/4}+DGT^{1/4},
\end{equation*}
\end{small}
where $D$, $G$ represent an upper bound on the diameter of $\mathcal{K}$ and an upper bound on the norm of the sub-gradients of $f_t(\bm{C})$ over $\mathcal{K}$, i.e. $\|\nabla f(\bm{C})\|\leq G,~\forall\,\bm{C}\in \mathcal{K}$.
\end{specialthm}
The proof of this theorem largely follows from that in~\cite{hazan2016introduction} with some variations. Due to the page limit, we omit it here. 
\section{Experiments}
\label{Experiments}
In this section, we evaluate the performance of the proposed online compact convexified factorization machine on two popular machine learning tasks: online rating prediction for recommendation and online binary classification.

\subsection{Experimental Setup}


\textbf{Compared Algorithms}\quad
We compare the empirical performance of OCCFM with state-of-the-art variants of FM in the online learning setting.
As the previous studies on FM focus on batch learning settings, we construct the baselines by applying online learning approaches to the existing formulations of FM, which is a common experimental methodology in online learning research \cite{Zhao:2011:OAM:3104482.3104512,li2014online}.
In the experiments, the comparison between OCCFM and other algorithms is focused on the aspects of formulation and online learning algorithms respectively. In existing research, two formulations of FM are related to OCCFM, i.e. the non-convex formulation of vanilla FM \cite{rendle2010factorization} and the non-compact formulation of Convex FM \cite{Yamada:2017:CFM:3097983.3098103}. Meanwhile the most popular online learning algorithms are Online Gradient Descent (OGD) \citep{Zinkevich:2003:OCP:3041838.3041955}, Passive-Aggressive (PA) \cite{Crammer:2006:OPA:1248547.1248566} and Follow-The-Regularized-Leader (FTRL) \cite{ad41159} which achieves the best empirical performance among them in most cases. To illustrate the comparison between different formulations, we apply the state-of-art FTRL algorithm on vanilla FM, CFM and the proposed CCFM, which are denoted as FM-FTRL, CFM-FTRL and CCFM-FTRL respectively. To illustrate the comparison between different online learning algorithms, we apply OGD, PA and FTRL on the proposed CCFM, which are denoted as CCFM-OGD, CCFM-PA, CCFM-FTRL respectively. 
To summarize, the compared algorithms are:
\begin{itemize}[leftmargin=*]
  \item FM-FTRL: vanilla FM with FTRL algorithm (a similar approach is proposed in \cite{ta2015factorization} for batch learning);
  \item CFM-FTRL: Convex FM \cite{Yamada:2017:CFM:3097983.3098103} with FTRL algorithm;
  \item CCFM-OGD: Compact Convexified FM with OGD algorithm;
  \item CCFM-PA: Compact Convexified FM with PA algorithm;
  \item CCFM-FTRL: Compact Convexified FM with FTRL algorithm;
  \item OCCFM: Online Compact Convexified FM algorithm.



\end{itemize}

\noindent\textbf{Datasets}\quad We select different datasets for the tasks respectively. For the Online Recommendation tasks, we use the typical Movielens datasets, including Movielens-100K, Movielens-1M and Movielens-10M $ ^5$\footnote{$ ^5$ https://grouplens.org/datasets/movielens/}; for the Online Binary Classification tasks, we select datasets from LibSVM $ ^6$, including IJCNN1, Spam and Epsilon \footnote{$ ^6$ https://www.csie.ntu.edu.tw/~cjlin/libsvmtools/datasets/binary.html}. The statistics of the datasets are summarized in Table \ref{dataset}.

\begin{small}
\begin{table}[!htbp]
\centering
\caption{Statistical Details of the Datasets}\label{dataset}
\scalebox{1.0}
{\begin{tabular}{|c|c|c|c|} \hline
\textbf{Datasets}&\textbf{\#Features}&\textbf{\#Instances}&\textbf{Label}\\ \hline\hline
\textbf{Movielens-100K} & 2,625 & 100,000 & Numerical\\
\textbf{Movielens-1M} & 9,940 & 1,000,209 & Numerical\\
\textbf{Movielens-10M} & 82,248 & 10,000,000 & Numerical\\ \hline\hline
%
\textbf{IJCNN1} & 22 & 141,691 & Binary\\
\textbf{Spam} & 252 & 350,000 & Binary\\
\textbf{Epsilon} & 2,000 & 100,000 & Binary\\ \hline

\hline\end{tabular}}
\end{table}
\end{small}

%

For each dataset, we conduct the experiments with five-fold cross-validation. In our experiments, the training instances are randomly permutated and fed one by one to the model sequentially. Upon the arrival of each instance, the model makes the prediction and
is updated after the label is revealed.
The experiment is conducted with 20 runs of different random permutations for the training data. The results are reported with the averaging performance over these runs.

\noindent\textbf{Evaluation Metrics}\quad To evaluate the performances on both tasks properly, we select different metrics respectively: the Root Mean Square Error (RMSE) for the rating prediction tasks; and the Error Rate and AUC (Area Under Curve) for the binary classification tasks, since AUC is known to be immune to the class imbalance problems.


%


\subsection{Online Recommendation}
In the online Recommendation task, at each round, the model receives a pair of user ID and item ID sequentially and then predicts the value of the incoming rating correspondingly. Denote the instance arriving at round $t$ as $(u_t,\,i_t,\,y_t)$, where $u_t$, $i_t$ and $y_t$ represent the user ID, item ID and the rating given by user $u_t$ to item $i_t$, the input feature vector $\bm{x}_t^T$ is constructed like this:
\begin{figure}[!htbp]
\vspace{-10pt}
\centering
\includegraphics[height=0.55in, width=2.4in]{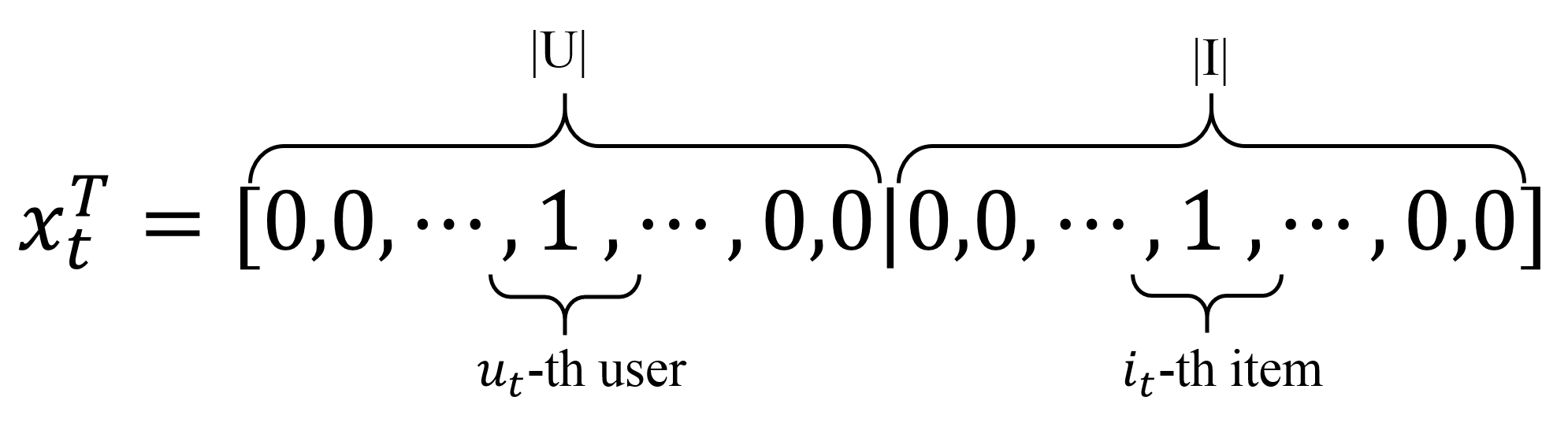}
\vspace{-10pt}
\end{figure}
\begin{figure}[!t]
\centering
\includegraphics[height=1.3in, width=3.4in]{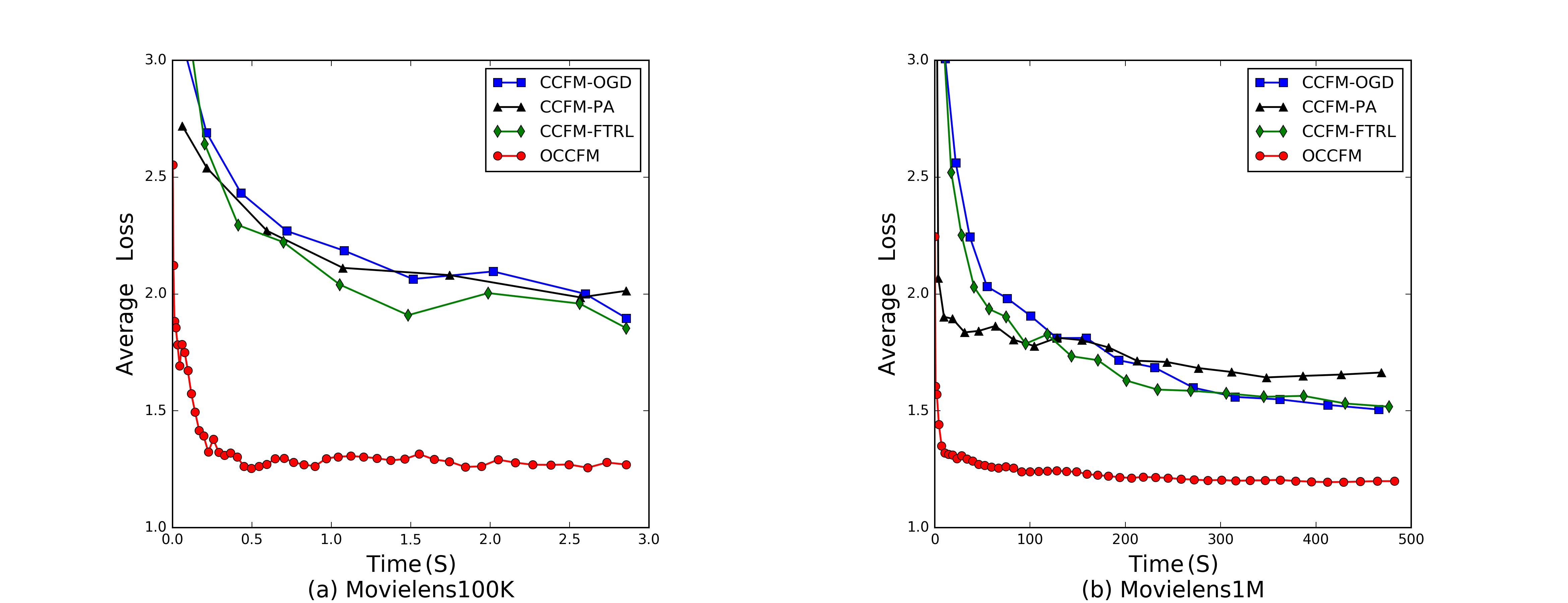}
\caption{Comparison of the efficiency of CCFM-OGD, CCFM-PA, CCFM-FTRL and OCCFM on Movielens 100K and Movielens 1M}
\label{fig:ave_loss_compare}
\end{figure}

\noindent where $|U|$ and $|I|$ refer to the number of users and the number of items respectively.
%
Upon the arrival of each instance, the model predicts the rating with $\hat{y}_t=\frac{1}{2}\hat{\bm{x}}_t^T\bm{C}_t\hat{\bm{x}}_t$, where $\hat{\bm{x}}_t=[\bm{x}_t^T,1]^T$.
The rating prediction task is essentially a regression problem, thus the convex loss function incurred at each round is the squared loss:
\begin{small}
\begin{equation*}
f_t(\bm{C}_t)=\|\hat{y}_t(\bm{C}_t)-y_t\|_2^2.
\end{equation*}
\end{small}

\noindent The nuclear norm bounds in the CCFM and CFM are set to $10$, $10$ and $20$ for Movielens-100K, 1M and 10M respectively. 
For FM, the rank parameters are set to $10$ on all the datasets. These hyper-parameters are selected with a grid search.
For OGD, PA and FTRL algorithms, the learning rate at round $t$ is set to $\frac{1}{\sqrt{t}}$ as what their corresponding theories suggest \cite{Crammer:2006:OPA:1248547.1248566,Zinkevich:2003:OCP:3041838.3041955,Shalev-Shwartz:2012:OLO:2185819.2185820}. In experiments on ML-1M and ML-10M, we randomly sample 1000 users and 1000 items for simplicity.


\begin{small}
\begin{table}[!htbp]
\renewcommand{\arraystretch}{1.2}
\centering
\caption{RMSE on Movielens-100K, Movielens-1M and Movielens-10M datasets in Online Rating Prediction of Recommendation tasks \protect\footnotemark[2]$ ^7$}
\label{tableRMSE}
\scalebox{0.9}{
\begin{tabular}{|c|c|c|c|}
  \hline
  \tabincell{c}{\multirow{2}{*}{\textbf{Algorithms}}} & \multicolumn{3}{c|}{\textbf{RMSE on Datasets}}\\
  \cline{2-4}
   & \textbf{Movielens100K} & \textbf{Movielens1M} & \textbf{Movielens10M} \\
  \hline\hline
  \textbf{FM-FTRL} & 1.2781 & 1.1074 & 1.0237 \\
  \textbf{CFM-FTRL} & 1.1036 & 1.0552 & 1.0078 \\ \hline\hline
  \textbf{CCFM-OGD} & 1.2721 & 1.0645 & 1.0291 \\
  \textbf{CCFM-PA} & 1.2164 & 1.0570 & 1.0142  \\
  \textbf{CCFM-FTRL} & $1.0873^{\dagger}$ & $1.0441^{\dagger}$ & $0.9725^{\dagger}$ \\ \hline\hline
  \textbf{OCCFM} & \textbf{1.0359*} & \textbf{0.9702*} & \textbf{0.9441*} \\
  \hline
\end{tabular}
}
\end{table}
\end{small}
\footnotetext[2]{$ ^7$ The results with $\dagger$ mark have passed the significance test with $p<0.01$ compared with FM-FTRL and CFM-FTRL; the results with $*$ mark have passed the significance test with $p<0.01$ compared with the other algorithms}

We list the RMSE of OCCFM and other compared algorithms in Table. \ref{tableRMSE}. From our observation, OCCFM achieves higher prediction accuracy than the other online learning baselines. Since FM-FTRL, CFM-FTRL, CCFM-FTRL use the same online learning algorithm with different formulations of FM, the comparison between them illustrates the advantage of Compact Convexified FM. Meanwhile, CCFM-OGD, CCFM-PA, CCFM-FTRL and OCCFM adopt the same formulation of Compact Convexified FM, the comparison between them shows the effectiveness of the conditional gradient step in OCCFM algorithm. 

We also measure the efficiency of CCFM-OGD, CCFM-PA, CCFM-FTRL and OCCFM algorithms on different datasets and see how fast the average losses decrease with the running time, which is shown in Fig. \ref{fig:ave_loss_compare}. From the results we can clearly observe that OCCFM runs significantly faster than CCFM-OGD, CCFM-PA and CCFM-FTRL, which illustrates the necessity and efficiency of using the linear optimization instead of the projection step.

\subsection{Online Binary Classification}
In the online binary classification tasks, the instances are denoted as $(\bm{x}_t,~y_t)\,\forall\,t$, where $\bm{x}_t$ is the input feature vector and $y_t\in\{-1,~+1\}$ is the class label. At round $t$, the model predicts the label with $sign(\hat{y}_t)=sign(\frac{1}{2}\hat{\bm{x}}_t^T\bm{C}_t\hat{\bm{x}}_t)$, where $\hat{\bm{x}}_t=[\bm{x}_t^T,1]^T$.
The loss function is a logistic loss function with respect to $\bm{C}_t$:

\begin{small}
\begin{equation*}
f_t(\bm{C_t})=log(1+\frac{1}{exp(-{y}_t\cdot\hat{y}_t(\bm{C}_t))}).
\end{equation*}
\end{small}

\noindent The nuclear norm bounds for CCFM and CFM in this task are set to $300$, $1000$, and $200$ for IJCNN1, Spam and Epsilon respectively; while the ranks for FM are set to $20$, $50$ and $50$ accordingly. All the hyper-parameters are selected with a grid search and set to the values with the best performances.
For OGD, PA and all FTRL algorithms, the learning rate at round $t$ is set to $\frac{1}{\sqrt{t}}$ as the theories suggest \cite{Shalev-Shwartz:2012:OLO:2185819.2185820}.
\begin{small}
\begin{table}[!htbp]
\renewcommand{\arraystretch}{1.0}
\centering
\caption{Error Rate and AUC on IJCNN1, Spam and Epsilon datasets in Online Binary Classification Tasks}
\label{tableER}
\scalebox{0.87}{
\begin{tabular}{|c|c|c|c|c|c|c|}
  \hline
  \tabincell{c}{\multirow{2}{*}{\textbf{Algorithms}}} & \multicolumn{3}{c|}{\textbf{Error Rate on Datasets}}& \multicolumn{3}{c|}{\textbf{AUC on Datasets}}\\
  \cline{2-7}
   & \textbf{IJCNN} & \textbf{Spam} & \textbf{Epsilon} & \textbf{IJCNN} & \textbf{Spam} & \textbf{Epsilon} \\ \hline\hline
  \textbf{FM-FTRL} & 0.0663 & 0.0867 & 0.1773 & 0.9647 & 0.9070 & 0.8213\\ 
  \textbf{CFM-FTRL} & 0.0544 & 0.0598 & 0.1654 & 0.9736 & 0.9390 & 0.8319\\ \hline\hline
  \textbf{CCFM-OGD} & 0.0911 & 0.1076 & 0.1742 & 0.9192 & 0.9239 & 0.8277 \\ 
  \textbf{CCFM-PA} & 0.0711 & 0.0976 & 0.1412 & 0.9312 & 0.9201 & 0.8575\\ 
  \textbf{CCFM-FTRL} & $0.0520^{\dagger}$ & $0.0588^{\dagger}$ & $0.1380^{\dagger}$ & 0.9757 & 0.9398 & $0.8668^{\dagger}$\\ \hline\hline
  \textbf{OCCFM} & \textbf{0.0243*} & \textbf{0.0567*} & \textbf{0.1202*} & \textbf{0.9859*} & \textbf{0.9414*} & \textbf{0.8737*}\\ \hline
\end{tabular}
}
\end{table}
\end{small}

The comparison of the prediction accuracy between OCCFM and other baseline algorithms is presented in Table \ref{tableER}. As shown in the table, OCCFM achieves the lowest error rates and the highest AUC values among all the online learning approaches, which reveals the advantage of OCCFM in prediction accuracy.

We also compare the running time of OCCFM and other baseline algorithms and present the results in Fig \ref{fig:ave_loss_compare_bc}. Similar with the observation in recommendation tasks, our proposed OCCFM runs significantly faster than CCFM-OGD and CCFM-FTRL.

\section{Related Work}
\label{Related_Work}

\begin{figure}[!t]
\centering
\includegraphics[height=1.3in, width=3.4in]{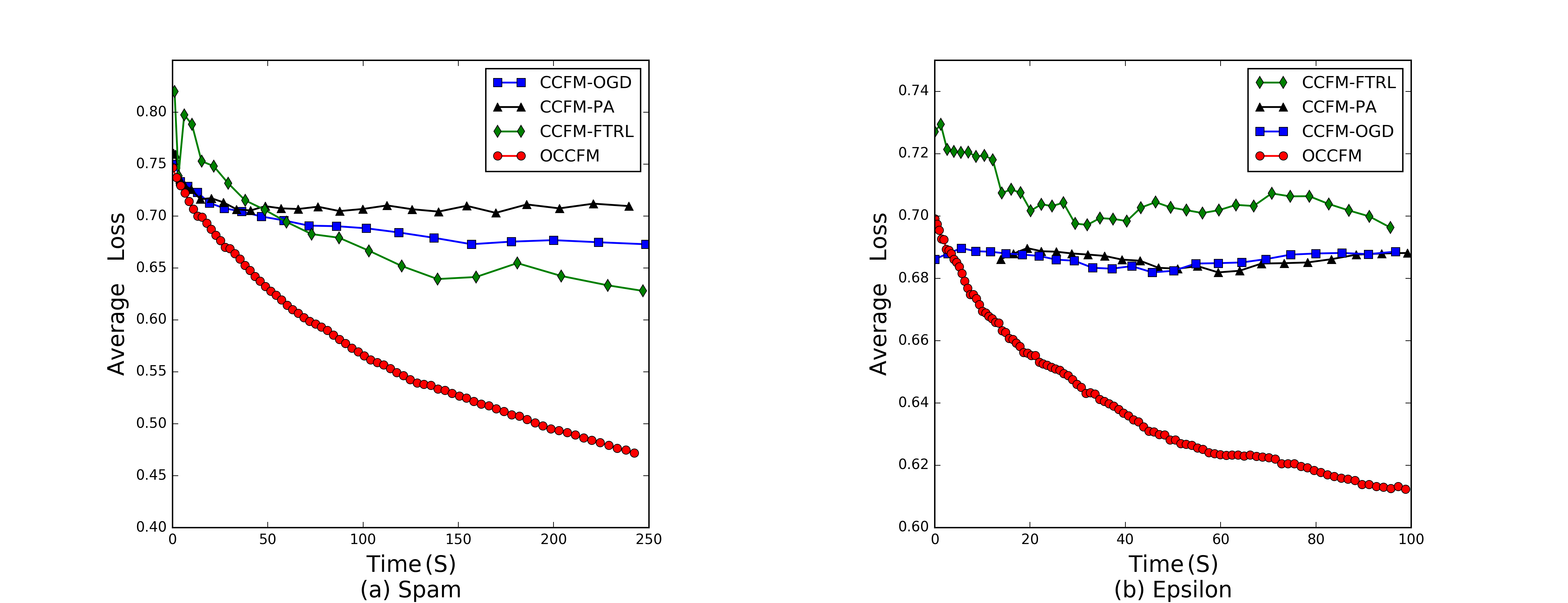}
\caption{Comparison of the efficiency of CCFM-OGD, CCFM-PA, CCFM-FTRL and OCCFM on Spam and Epsilon}
\label{fig:ave_loss_compare_bc}
\end{figure}

\subsection{Factorization Machine}
The core of FM is to leverage feature interactions for feature augmentation. According to the order of feature interactions, the existing research can be categorized into two lines: the first category \cite{juan2017field, cheng2014gradient, Yuan:2017:BBF:3025171.3025211}  focuses on second-order FM which models pair-wise feature interactions for feature engineering; while the second category \cite{blondel2016polynomial, blondel2016higher} attempts to model the interactions between arbitrary number of features. One line of related research in the first category is Convex Factorization Machine \cite{blondel2015convex, yamada2015convex, Yamada:2017:CFM:3097983.3098103}, which looks for convex variants of FM. However, the formulations of Convex FM are not well organized into a compact form which we provide for Compact Convexified FM to meet the requirements of online learning setting.

The most significant difference between OCCFM and previous research is that OCCFM is an online machine learning model while most existing variants are batch learning models. Some studies attempt to apply FM in online applications \cite{kitazawa2016incremental} \cite{ta2015factorization}. But these studies do not fit in the online learning framework with theoretical guarantees while OCCFM provides a theoretically provable regret bound.

\subsection{Online Learning}
Online learning stands for a family of efficient and scalable machine learning algorithms \cite{Cesa2004On, Crammer:2006:OPA:1248547.1248566, Rosenblatt1958The, Hoi2014LIBOL, Wang:2016:SCL:2973184.2932193}. Unlike conventional batch learning algorithms, the online learning algorithms are built upon the assumption that the training instances arrive sequentially rather than being available prior to the learning task.

Many algorithms have been proposed for online learning, including the classical Perception algorithm and Passive-Aggressive (PA) algorithm \cite{Crammer:2006:OPA:1248547.1248566}. In recent years, the design of many efficient online learning algorithms has been influenced by convex optimization tools \cite{Dredze:2008:CLC:1390156.1390190, Wang:2016:SCL:2973184.2932193}. Some typical algorithms include Online Gradient Descent (OGD) and Follow-The-Regularized-Leader (FTRL) \cite{abernethy2008competing} \cite{shalev2012online}. However, their further applicability is limited by the expensive projection operation required when additional norm constraints are added.
Recently, the Online Conditional Gradient (OCG) algorithm \cite{pmlr-v28-jaggi13} has regained a surge of research interest. It eschews the computational expensive projection operation thus is highly efficient in handling large-scale learning problems. Our proposed OCCFM model builds upon the OCG algorithm. For further details, please refer to \cite{pmlr-v28-jaggi13} and \cite{hazan2016introduction}.

\section{Conclusion}
\label{Conclusion}
In this paper, we propose an online variant of FM which works in online learning settings. The proposed OCCFM meets the requirements in OCO and is also more cost effective and accurate compared to extant state-of-the-art approaches. In the study, we first invent a convexification Compact Convexified FM (CCFM) based on OCO such that it fulfills the two fundamental requirements within the OCO framework. Then, we follow the general projection-free algorithmic framework of Online Conditional Gradient and propose an Online Compact Convex Factorization Machine (OCCFM) algorithm that eschews the projection operation with linear optimization step. In terms of theoretical support, we prove that the algorithm preserves a sub-linear regret, which indicates that our algorithm can perform as well as the best fixed algorithm in hindsight. Regarding the empirical performance of OCCFM, we conduct extensive experiments on real-world datasets. The experimental results on recommendation and binary classification tasks indicate that OCCFM outperforms the state-of-art online learning algorithms.

\bibliographystyle{ACM-Reference-Format}
\bibliography{OCFM}
\end{document}